\documentclass{article}

\usepackage[letterpaper,top=2cm,bottom=2cm,left=3cm,right=3cm]{geometry}
\usepackage{amsmath,amssymb,amsthm,amsfonts}
\usepackage{graphicx}
\usepackage[colorlinks=true, allcolors=blue]{hyperref}
\usepackage{bm}
\usepackage{tikz}
\usepackage{authblk}

\newtheorem{theorem}{Theorem}[section]
\newtheorem{lemma}[theorem]{Lemma}

\newtheorem{assumption}[theorem]{Assumption}
\newtheorem{proposition}[theorem]{Proposition}

\providecommand{\keywords}[1]
{
	\small	
	\textbf{\textit{Keywords: }} #1
}

\newcommand{\y}{\bm{y}}

\DeclareMathOperator{\trace}{trace}
\DeclareMathOperator{\rank}{rank}

\numberwithin{equation}{section}

\title{ Geometric Analysis of Unconstrained Feature Models with $d=K$
 } 
\author{Shao Gu \thanks{202220101012@mails.zstu.edu.cn}  and Yi Shen \thanks{yshen@zstu.edu.cn}}
\affil{
Department of Mathematics, 
Zhejiang Sci-Tech University}

\begin{document}

\maketitle

\begin{abstract}

Recently, interesting empirical phenomena known as Neural Collapse have been observed during the final phase of training deep neural networks for classification tasks. We examine this issue when the feature dimension $d$ is equal to the number of classes $K$. We demonstrate that two popular unconstrained feature models are strict saddle functions, with every critical point being either a global minimum or a strict saddle point that can be exited using negative curvatures. The primary findings conclusively confirm the conjecture on the unconstrained feature models in \cite{Zhou2022OnTO,zhu2021A}.

\end{abstract}

\section{Introduction}
Consider a classification task with $K$ classes and $n$ training
samples per class, i.e., $N = nK$ samples.
The weight of the
final linear classifier and bias are denoted respectively by 
$\bm{W} \in  \mathbb{R}^{K\times d} $ and $ \bm{b} \in \mathbb{R}^{K}$.
Let   
$\bm{h}_{k,j} \in \mathbb{R}^{d}$
denote the last-layer feature vector of the $j$-th training sample
of the $k$-th class. Let
$\y_k\in \mathbb{R}^K$ denote the
corresponding label, which is the one-hot vector with one in its $k$-th
entry. We are interested in two well-known loss functions utilized in machine learning. One is
the cross-entropy loss as indicated by the subsequent form
\begin{equation}\label{celoss}
    \mathcal{L}_{CE}(\bm{W}\bm{h}_{k,j}+\bm{b},\bm{y}_{k})=\log \left(\frac{\sum_{l=1}^{K}e^{(\bm{W}\bm{h}_{k,j}+\bm{b})^{\top}\bm{y}_{l}}}{e^{{\bm{w}^{k}}\bm{h}_{k,j}+b_{k}}}\right),
    \quad
    k \in [K], \quad j \in [n],
\end{equation}
while the other is the mean squared error loss given by
\begin{equation}\label{mseloss}
  \mathcal{L}_{MSE}\left(\bm{WH+b1}^{\top}_{N}-\bm{Y}\right)=
  \|\bm{WH+b1}^{\top}_{N}-\bm{Y}\|_{F}^{2},
\end{equation}
where 
\begin{equation*}
     \bm{Y}=\begin{pmatrix}
         \bm{y}_{1}& \cdots & \bm{y}_{1} & \bm{y}_{2} &  \cdots & \bm{y}_{2} & \cdots & \bm{y}_{K} & \cdots \bm{y}_{K}  
    \end{pmatrix} \in \mathbb{R}^{K\times N}
\end{equation*}
and 
$$
\bm{H}=
\begin{pmatrix}
\bm{h}_{1,1} &  \cdots & \bm{h}_{1,n} & \bm{h}_{2,1} &  \cdots & \bm{h}_{2,K} &\cdots  &\bm{h}_{K,1} &  \cdots  &\bm{h}_{K,n}
\end{pmatrix}
 \in \mathbb{R}^{d\times N}.
$$ 
Afterwards, network parameters can be obtained by minimizing the unconstrained features model with cross-entropy loss as follows
\begin{equation}\label{maince}
\underset{\bm{W},\bm{H},\bm{b}}{\min} 
f^{C}(\bm{W},\bm{H},\bm{b}) 
=
\frac{1}{N}\sum_{k=1}^{K}\sum_{j=1}^{n} \mathcal{L}_{CE}(\bm{W}\bm{h}_{k,j}+\bm{b},\bm{y}_{k})+\frac{\lambda_{\bm{W}}}{2}\| \bm{W}\|^{2}_{F}+\frac{\lambda_{\bm{H}}}{2}\| \bm{H}\|^{2}_{F}+\frac{\lambda_{\bm{b}}}{2}\| \bm{b}\|^{2}_{2},
\end{equation}
or by minimizing the unconstrained features model with mean squared error loss as follows
\begin{equation}\label{mainmse}
\underset{\bm{W},\bm{H},\bm{b}}{\min} 
f^{M}(\bm{W},\bm{H},\bm{b}) 
=
\frac{1}{2N}\mathcal{L}_{MSE}\left(\bm{WH+b1}^{\top}_{N}-\bm{Y}\right)+\frac{\lambda_{\bm{W}}}{2}\| \bm{W}\|^{2}_{F}+\frac{\lambda_{\bm{H}}}{2}\| \bm{H}\|^{2}_{F}+\frac{\lambda_{\bm{b}}}{2}\| \bm{b}\|^{2}_{2}
\end{equation}
where
$\lambda_{\bm{W}}$,
$\lambda_{\bm{H}}$,
$\lambda_{\bm{b}}>0$ are the penalty parameters for the weight decay. 
When $d>k$, the
optimization landscapes of 
the model \eqref{maince} and the model \eqref{mainmse}
have been obtained in
\cite{Zhou2022OnTO}
and 
\cite{zhu2021A}, respectively.
In this paper, we study the global optimization landscape of  
the unconstrained features models \eqref{maince} and \eqref{mainmse} with $d=K$.
The following results, which called
no spurious local minima and strict saddle property, answer the  conjectures in \cite{Zhou2022OnTO,zhu2021A} 
positively.
\begin{theorem}\label{thm2} 
		Assume that the feature dimension $d$ is equal to the number of classes $K$.
	The function $f^{C}(\bm{W},\bm{H},\bm{b})$ 
	in \eqref{maince} 
	is a  strict saddle function with no spurious local minimum, in the sense that
	\begin{enumerate}
		\item  Any local minimizer of \eqref{maince} is a global minimizer  of \eqref{maince}.
		\item 
		Any critical point of \eqref{maince} that is not a local minimizer has at least one negative curvature direction, i.e., the Hessian 
		$\nabla^{2}f^{C}(\bm{W},\bm{H},\bm{b})$, at this critical point, is non-degenerate and has at least
		one negative eigenvalue.
	\end{enumerate}
\end{theorem}

\begin{theorem}\label{thm3} 
		Assume that the feature dimension $d$ is equal to the number of classes $K$.
	The function $f^{M}(\bm{W},\bm{H},\bm{b})$ 
	in \eqref{mainmse} 
	is a  strict saddle function with no spurious local minimum, in the sense that
	\begin{enumerate}
		\item  Any local minimizer of \eqref{mainmse} is a global minimizer  of \eqref{mainmse}
		\item 
		Any critical point of \eqref{mainmse} that is not a local minimizer has at least one negative curvature direction, i.e., the Hessian 
		$\nabla^{2}f^{M}(\bm{W},\bm{H},\bm{b})$, at this critical point, is non-degenerate and has at least
		one negative eigenvalue.
	\end{enumerate}
\end{theorem}

Deep learning usually utilizes a feature dimension $d$ that is substantially bigger than the number of classes $K$ in many classification problems. 
Theorem \ref{thm2} and  Theorem  \ref{thm3}, on the other hand, suggest that selecting a $d$ that is significantly larger than the number of classes $K$ is not necessary. Lowering the dimension $d$ can result in significant savings on computation and memory expenses.
For example, experiments in \cite{zhu2021A} show that one may set the feature dimension $d$ equal to the number
of classes and fix the last-layer classifier to be a Simplex Equiangular Tight Frame (ETF) for network training, which
reduces memory cost by over 20\% on ResNet18 without sacrificing the generalization
performance. 
Since
both unconstrained features models \eqref{maince} and  \eqref{mainmse} are strict saddle functions,
their global solutions  can be efficiently 
found by any method that can escape strict saddle points, such as gradient descent
with random initialization \cite{Lee2016}.
It was proved \cite{zhu2021A} that
the global solution $(\bm{W}^{\star},\bm{H}^{\star},\bm{b}^{\star})$  of \eqref{maince} obeys
\begin{equation}\label{eq:nc1}
	w^{\star}:=\| \bm{w}^{\star1}\|_{2}=\cdots=\| \bm{w}^{\star K}\|_{2},\quad \bm{b}^{\star}=b^{\star}\bm{1}_{K}, 
\end{equation}
\begin{equation}\label{eq:nc2}
\bm{h}^{\star}_{k,j}
=
\sqrt{\frac{\lambda_{\bm{W}}}{n\lambda_{\bm{H}}}}
{
(\bm{w}^{\star k})^{\top}},
\quad
\ k\in\left[K\right],
\  j \in\left[n\right],
\quad 
\bar{\bm{h}}^{\star}_{j}:=
\frac{1}{K}\sum_{k=1}^{K}\bm{h}^{\star}_{k,j}=\bm{0},\quad j\in\left[n\right]
\end{equation}
where either $b^{\star}=0$ or $\lambda_{\bm{b}}=0$. Moreover,
the matrix $\bm{W}^{\star\top}$ forms a $K$-Simplex  ETF up to some scaling and rotation, in the sense that for any $\bm{U}\in \mathbb{R}^{K\times K}$ with $\bm{U}^{\top}\bm{U}=\bm{I}_{K}$, the normalized
matrix $\bm{M}:=\frac{1}{w^{\star}}\bm{U}\bm{W}^{\star\top}$  satisfies 
\begin{equation}\label{eq:nc3}
\bm{M}^{\top}\bm{M}=\frac{K}{K-1}\left(\bm{I}_{K}-\frac{1}{K}\bm{1}_{K}\bm{1}_{K}^{\top}\right).
\end{equation}
The ETF consists of unit vectors with equal lengths and maximally separated pair-wise angles \cite{Welch}. This property makes ETF useful in various signal processing and data analysis applications, including compressed sensing, quantum information theory, and coding theory \cite{FICKUS2020130, Malozemov, STROHMER2003257}.
Properties \eqref{eq:nc1}, \eqref{eq:nc2}, and \eqref{eq:nc3}, called neural collapse, are interconnected features of the final layer classifiers and features in
the terminal phase of training deep neural networks used for classification tasks \cite{han2022neural, papyan_prevalence_2020}. Neural collapse suggests that the network aims to maximize the angular differences between each class and its corresponding classifier.
The global solutions of mean squared error loss also exhibit the neural collapse phenomenon,
as discussed in \cite{Zhou2022OnTO}. More theoretical explorations on neural collapse can be found in \cite{ji2022an, Kothapalli2022NeuralCA, LU2022224, mixon_neural_2022, Zhou2022OnTO, zhu2021A}, and other references therein.

\section{Notation}
This section presents the notation that are used in this paper throughout.
Matrices and vectors are
denoted in boldface such as $\bm{Z}$ and $\bm{z}$. 
The transposes of $\bm{Z}$ and $\bm{z}$, respectively, are denoted by the symbols $\bm{Z}^{\top}$ and $\bm{z}^{\top}$ for both matrices and vectors.
Let $\bm{z}_i$ and $\bm{z}^i$ represent the column and row vectors of a specified matrix $\bm{Z}$.
Normal typeface is used to indicate the individual elements in a matrix or vector, such as $z_{ij}$ or $z_i$.
For any positive integer $K$, 
we use $[K] := \{1,2,\ldots, K\}$ to denote the  set of indices up to $K$.
The symbols $\bm{I}_K$ and $\bm{1}_K$
represent the identity matrix and the all-ones vector with an appropriate size of $K$, respectively.
Let $\bm{0}$ denote zero vectors or zero matrices whose dimensions are determined by context.
The  Euclidean norm of a vector is denoted by $\|\cdot\|_2$. The spectral norm, 
the Frobenius norm, the nuclear norm, the
trace, and the rank 
of a matrix are denoted  by
 $\| \cdot\|$,
 $\|\cdot\|_F$,
$\|\cdot\|_*$,
$\trace(\cdot)$, and 
$\rank(\cdot)$, respectively.
The compact singular value decomposition (SVD) of $\bm{Z}$ is defined as $\bm{Z}=\bm{U}\bm{\Sigma}\bm{V}^{\top}$.
For a given matrix  $\bm{Z}$   of size $K\times N$,
the partial derivative of a scalar function $f\left(\bm{Z}\right)$ with respect to 
$\bm{Z} $ 
is defined by
$\nabla f\left(\bm{Z}\right)   \in \mathbb{R}^{K\times N}$ whose
the $(k,j)$  entry is
$$
[\nabla f\left(\bm{Z}\right)]_{k,j} = \frac{\partial f\left(\bm{Z}\right)}{\partial z_{kj}},
\quad k\in [K], \quad j\in [N].
$$
Similarly, the partial derivative of a scalar function $f\left(\bm{Z},\bm{b}\right)$ with respect to 
$
\bm{Z} $ is denoted  by  
$
\nabla f_{\bm{Z}}\left(\bm{Z}, \bm{b}\right) $ whose  $(k,j)$ entry is  
$$
[\nabla f_{\bm{Z}}\left(\bm{Z},\bm{b}\right)]_{k,j} = \frac{\partial f\left(\bm{Z},\bm{b}\right)}{\partial z_{kj}},
\quad k\in [K], \quad j\in [N].
$$
The  Hessian matrix of a scalar function $f(\bm{Z})$  is represented by 
$$
{\nabla^2 f(\bm{Z})[\bm{A},\bm{B}]} = \sum_{k,j,k',j'} \frac{\partial^2 f(\bm{Z}) }{\partial z_{kj}z_{k'j'}}a_{kj}b_{k'j'},
\quad k,\ k'\in [K], \quad j,\ j'\in [N],
$$
for any $\bm{A}$, $\bm{B} \in \mathbb{R}^{K\times N}$.

\section{Proof of Theorem \ref{thm2}}

We first establish a lemma which
plays a key role in the proof of our main results.
For the variable  $(\bm{W}, \bm{H}, \bm{b})$ in \eqref{maince}, 
we define an auxiliary variable
$$
\bm{R=WH+b}\bm{1}_N^{\top} \in\mathbb{R}^{K\times N},$$ where  the $k$-th row and the $\left[(k'-1)n+j\right]$-th column of $\bm{R}$ is
$$
r_{k;k',j}={\bm{w}^{k}}\bm{h}_{k',j}+b_{k},\quad  k\in [K], \ k'\in [K],\ j\in [n].
$$ 	 
For the loss function in \eqref{maince},
we define an auxiliary function
\begin{equation}
	g(\bm{R})=	\frac{1}{N}\sum_{k=1}^{K}\sum_{j=1}^{n} \mathcal{L}_{CE}(\bm{W}\bm{h}_{k,j}+\bm{b},\bm{y}_{k}).
\end{equation}

\begin{lemma}\label{l3}
	Any critical point $\left(\bm{W},\bm{H},\bm{b}\right)$ of  \eqref{maince} obeys
	\begin{equation*}
		\bm{W}^{\top}\bm{W}=\frac{\lambda_{\bm{H}}}{\lambda_{\bm{W}}}\bm{HH}^{\top},
  \quad \rank(\bm{W})=\rank(\bm{H})\leq K-1.
	\end{equation*}
\end{lemma}
\begin{proof}

	Any critical point $\left(\bm{W},\bm{H},\bm{b}\right)$ of \eqref{maince} satisfies:
	\begin{equation}\label{r1}
		\nabla_{\bm{W}}
		f^{C}\left(\bm{W},\bm{H},\bm{b}\right)=
		\nabla g(\bm{R})\bm{H}^{\top}+\lambda_{\bm{W}}\bm{W}=\bm{0},
	\end{equation}
	\begin{equation}\label{r2}
		\nabla_{\bm{H}}
		f^{C}\left(\bm{W},\bm{H},\bm{b}\right)=\bm{W}^{\top}\nabla
		g(\bm{R})+\lambda_{\bm{H}}\bm{H}=\bm{0}.
	\end{equation}
It follows from
	\eqref{r1} and \eqref{r2} that
	$$	\bm{W}^{\top}\nabla g(\bm{R})\bm{H}^{\top}=-\lambda_{\bm{W}}\bm{W}^{\top}\bm{W},\quad 
	\bm{W}^{\top}\nabla g(\bm{R})\bm{H}^{\top}=-\lambda_{\bm{H}}\bm{HH}^{\top}.
	$$	Thus,
	\begin{equation}\label{eq:CEWH}
	    	\lambda_{\bm{W}}\bm{W}^{\top}\bm{W}=\lambda_{\bm{H}}\bm{HH}^{\top}.
	\end{equation}

Moreover, it follows from  \eqref{r1} that
	\begin{align}
		\rank(\bm{{W}})
		&=\rank(-\lambda_{\bm{W}}\bm{W})   \nonumber \\
		&=\rank(\nabla g(\bm{R})\bm{H}^{\top})  \nonumber \\
		&\leq\min\left\{ \rank(\bm{H}^{\top}),\rank\left(\nabla g(\bm{R})\right)\right\}. \label{eq:whmaince}
	\end{align}
Similarly, it follows from \eqref{r2} that
	\begin{equation}\label{eq:wh4}
		\rank(\bm{H})\leq\min\left\{\rank(\bm{W}^{\top}),\rank(\nabla g(\bm{R}))\right\}.
	\end{equation}
Combining \eqref{eq:whmaince} and \eqref{eq:wh4}, we have
	\begin{equation}\label{eq:wh2}	\rank(\bm{W})=\rank(\bm{H})\leq\rank(\nabla g(\bm{R})).
	\end{equation}
	The elements of  $\nabla g(\bm{R}) \in\mathbb{R}^{K\times N}$ are  given by
	\begin{align*}
		\frac{\partial g(\bm{R})}{\partial r_{k;k,j}}&=
		\frac{1}{N}\left(
		-1+\frac{e^{r_{k;k,j}}}{\sum_{l=1}^{K}e^{r_{l;k,j}}}\right),\quad k\in[K], \ j \in [n],\\
		\frac{\partial g(\bm{R})}{\partial r_{k;k',j}}&=
		\frac{1}{N}\left(
		\frac{e^{r_{k;k',j}}}{\sum_{l=1}^{K}e^{r_{l;k',j}}}\right),\quad  k'\neq k,\ \ k \in [K],\  k'\in[K], \ j \in [n].
	\end{align*}
	Direct calculation yields 
	\begin{align*}
		&\bm{1}_K^{\top}\nabla g(\bm{R}) \\
		=&\left(\frac{\partial g(\bm{R})}{\partial r_{1;1,1}}+\sum_{k\neq 1}\frac{\partial g(\bm{R})}{\partial r_{k;1,1}},\cdots,
		\frac{\partial g(\bm{R})}{\partial r_{K;K,n}}+\sum_{k\neq K}\frac{\partial g(\bm{R})}{\partial r_{k;K,n}}\right)\\
		=&\frac{1}{N}\left(\left(-1+\frac{e^{r_{1;1,1}}}{\sum_{l=1}^{K}e^{r_{l;1,1}}}\right)+\sum_{k\neq 1}\frac{e^{r_{k;1,1}}}{\sum_{l=1}^{K}e^{r_{l;1,1}}},\cdots,\left(-1+\frac{e^{r_{K;K,n}}}{\sum_{l=1}^{K}e^{r_{l;K,n}}}\right)+\sum_{k\neq K}\frac{e^{r_{k;K,n}}}{\sum_{l=1}^{K}e^{r_{l;K,n}}}\right)\\
		=&\left(0,\cdots,0\right )\in\mathbb{R}^{1\times N}.
	\end{align*}
	This indicates 
	\begin{equation}\label{eq:rk1}
		\rank(\nabla g(\bm{R}))\leq K-1.
	\end{equation}
	It follows from {\eqref{eq:wh2}} and \eqref{eq:rk1} that
	\begin{equation*}
		\rank(\bm{W})=\rank(\bm{H})\leq \rank(\nabla g(\bm{R})) \le K-1.
	\end{equation*}
\end{proof}

\begin{proof}[Proof of Theorem \ref{thm2}]\label{pot1}
	The critical points of $f^{C}(\bm{W},\bm{H},\bm{b}) $ are gathered by
	$$
	\mathcal{C}:=\left\{(\bm{W},\bm{H},\bm{b}) \vert \nabla_{\bm{W}} f^{C}(\bm{W},\bm{H},\bm{b})
	=\bm{0},\ \nabla_{\bm{H}} f^{C}(\bm{W},\bm{H},\bm{b})  =\bm{0},\ \nabla_{\bm{b}} f^{C}(\bm{W},\bm{H},\bm{b})  =\bm{0}\right\}.
	$$	
	We	separate the set $\mathcal{C}$ into two disjoint subsets 
	\begin{align*}
	\mathcal{C}_{1}:&= \mathcal{C} \cap \left\{(\bm{W},\bm{H},\bm{b}) \vert 
	\|\nabla g\left(\bm{WH+b}\bm{1}_N^{\top}\right)\|\leq \sqrt{\lambda_{\bm{W}}\lambda_{\bm{H}}}\right\} ,\\
	\mathcal{C}_{2}:&= \mathcal{C} \cap \left\{(\bm{W},\bm{H},\bm{b}) \vert \|\nabla g\left(\bm{WH+b}\bm{1}_N^{\top}\right)\|> \sqrt{\lambda_{\bm{W}}\lambda_{\bm{H}}}\right\},
	\end{align*}
	satisfying $\mathcal{C}=\mathcal{C}_{1}\cup \mathcal{C}_{2}$. 
	It follows from \cite[Lemma C.4]{zhu2021A}  that any $(\bm{W},\bm{H},\bm{b}) \in\mathcal{C}_{1}$ is a global optimal solution of $f^{C}(\bm{W},\bm{H},\bm{b}) $ in \eqref{maince}. 
	In the rest of proof, we show any vector in $\mathcal{C}_{2}$ possesses negative curvatures.
	For any direction 
	$$
	\bm{\Delta}=\left(\bm{\Delta_{W}},\bm{\Delta_{H}},\bm{\Delta_{b}}\right),
	$$  
	the Hessian bilinear form of $f^{C}(\bm{W},\bm{H},\bm{b}) $ along the direction $\bm{\Delta}$ is
	\begin{equation}\label{hb}
	\begin{split}
	\begin{aligned}        
	&\nabla^{2}f^{C}(\bm{W},\bm{H},\bm{b}) \left[\bm{\Delta},\bm{\Delta}\right]\\ 
	=&
	\nabla^{2}g\left(\bm{WH+b}\bm{1}_N^{\top}\right)
	\left[
	\bm{W\Delta_{H}+\Delta_{W}H+\Delta_{b}\bm{1}_N^{\top}},
	\bm{W\Delta_{H}+\Delta_{W}H+\Delta_{b}\bm{1}_N^{\top}}
	\right]\\
	&
	+2\trace\left[\left(\nabla g\left(\bm{WH+b}\bm{1}_N^{\top}\right)\right)^{\top}\bm{\Delta_{W}\Delta_{H}}\right]
	+\lambda_{\bm{W}}\|\bm{\Delta_{W}}\|_{F}^{2}
	+\lambda_{\bm{H}}\|\bm{\Delta_{H}}\|_{F}^{2}
	+\lambda_{\bm{b}}\|\bm{\Delta_{b}}\|_{2}^{2}.
	\end{aligned}
	\end{split}
	\end{equation}
	The necessary and sufficient condition of $\nabla^{2}f^{C}(\bm{W},\bm{H},\bm{b}) $ has at least one negative eigenvalue can be  expressed as  
	$$
	\nabla^{2}f^{C}(\bm{W},\bm{H},\bm{b})\left[\bm{\Delta},\bm{\Delta}\right]<0, 
	$$
	with some certain direction $\bm{\Delta}$.

	By Lemma \ref{l3}, we know that $\rank\left(\bm{W}\right)\leq K-1$. Hence, there exists a nonzero unit vector $\bm{a}\in\mathbb{R}^{K}$ in the null space of $\bm{W}$, i.e., 
	$$
	\bm{Wa}=\bm{0}.
	$$
	This, together with  \eqref{eq:CEWH}, implies that 

 	\begin{equation*}
0=\|\bm{Wa}\|_{2}^{2}
=\bm{a}^{\top}\bm{W}^{\top}\bm{Wa}
=\frac{\lambda_{\bm{H}}}{\lambda_{\bm{W}}}\bm{a}^{\top}\bm{HH}^{\top}\bm{a}
=\frac{\lambda_{\bm{H}}}{\lambda_{\bm{W}}}\|\bm{H}^{\top}\bm{a}\|_{2}^{2},
	\end{equation*}
 which leads to
 $$
 \bm{H^{\top}a} = \bm{0}.
 $$
The left and right singular vectors to the largest singular value of $\nabla g\left(\bm{WH+b}\bm{1}_N^{\top}\right)$ is denoted by $\bm{u}$ and $\bm{v}$, i.e., 
	$$
	\bm{u}^{\top}\nabla g\left(\bm{WH+b}\bm{1}_N^{\top}\right)\bm{v}=
	\left\|\nabla g\left(\bm{WH+b}\bm{1}_N^{\top}\right)\right\|.
	$$ 
	We construct the specific negative curvature direction
	\begin{equation}\label{eq:d}  
	\bm{\Delta}=\left(\bm{\Delta_{W}},\bm{\Delta_{H}},\bm{\Delta_{b}}\right)=
	\left(
	\left(\frac{\lambda_{\bm{H}}}{\lambda_{\bm{W}}}\right)^{\frac{1}{4}}\bm{ua^{\top}},
	-\left(\frac{\lambda_{\bm{H}}}{\lambda_{\bm{W}}}\right)^{-\frac{1}{4}}\bm{av^{\top}},
	\bm{0}
	\right).
	\end{equation}
	Since 
	$$\bm{Wa}= \bm{0},\quad \bm{a^{\top}H} = \bm{0}^{\top}, \quad \bm{\Delta_{b}}=\bm{0},$$
	direct calculation shows
	$$
	\bm{W}\bm{\Delta}_{\bm{H}}
	+\bm{\Delta}_{\bm{W}}\bm{H}+\bm{\Delta}_{\bm{b}}\bm{1}_N^{\top}
 =\bm{0},
	$$
	which leads to
	\begin{equation}\label{eq:zero}
	\nabla^{2}g\left(\bm{WH+b}\bm{1}_N^{\top}\right)
	\left[\bm{W\Delta_{H}+\Delta_{W}H+\Delta_{b}\bm{1}_N^{\top}},\bm{W\Delta_{H}+\Delta_{W}H+\Delta_{b}}\bm{1}_N^{\top}\right]= 0.
	\end{equation}
	It follows from \eqref{hb},  \eqref{eq:d} and \eqref{eq:zero} that
	\begin{align*}
	\nabla^{2}f^{C}(\bm{W},\bm{H},\bm{b})[\bm{\Delta,\Delta}]
	=&-2\trace\left(\left(\nabla g\left(\bm{WH+b}\bm{1}_N^{\top}\right)\right)^{\top}\bm{uv^{\top}}\right)+2\sqrt{\lambda_{\bm{W}}\lambda_{\bm{H}}}\\
	=&-2\left(\left\|\nabla g\left(\bm{WH+b}\bm{1}_N^{\top}\right)\right\|-\sqrt{\lambda_{\bm{W}}\lambda_{\bm{H}}}\right)<0,
	\end{align*}
	where the last inequality is implied by $\|\nabla g\left(\bm{WH+b}\bm{1}_N^{\top}\right)\|>\sqrt{\lambda_{\bm{W}}\lambda_{\bm{H}}}$.

\end{proof}

\section{Proof of Theorem \ref{thm3}}
We establish several propositions on the critical points of  \eqref{mainmse} first.
For the variable   in \eqref{mainmse}, 
we define an auxiliary variable 
 $$
 \widetilde{\bm{Y}}=\bm{Y-b}\bm{1}_N^{\top} \in\mathbb{R}^{K\times N}$$ 
and an auxiliary function
 $$
 \tilde{g}(\bm{WH+b1}_{N}^{\top})=\frac{1}{2N}\mathcal{L}_{MSE}\left(\bm{WH+b1}^{\top}_{N}-\bm{Y}\right).
 $$

\begin{proposition}\label{mselm2}
	If a critical point $\left(\bm{W},\bm{H},\bm{b}\right)$ of \eqref{mainmse} satisfies 
	\begin{equation}\label{mseopt}
	\|\bm{WH-\widetilde{Y}}\|\leq N\sqrt{\lambda_{\bm{W}}\lambda_{\bm{H}}},
	\end{equation}
	then it is a global minimizer of \eqref{mainmse}.
\end{proposition}
\begin{proof}
	
	Direct calculation yields
	$$
	\nabla\tilde{g}(\bm{WH+b1}_{N}^{\top})=
	\frac{1}{N}\left(\bm{WH-\widetilde{Y}}\right).
	$$
	It follows from \cite[Lemma C.4]{zhu2021A}  that  the critical point $(\bm{W},\bm{H},\bm{b})$ of \eqref{mainmse} satisfies 
	$$
	\frac{1}{N}
	\|\bm{WH-\widetilde{Y}}\|\leq \sqrt{\lambda_{\bm{W}}\lambda_{\bm{H}}}
	$$
	is a global minimizer of \eqref{mainmse}. 

\end{proof}
\begin{proposition}\label{mselm1}
    Any critical point $\left(\bm{W},\bm{H},\bm{b}\right)$ of  \eqref{mainmse} obeys
\begin{equation}\label{eq:wh1}
         		\bm{W}^{\top}\bm{W}=\frac{\lambda_{\bm{H}}}{\lambda_{\bm{W}}}\bm{HH}^{\top},
\quad \|\bm{W}\|_{F}^{2}=\frac{\lambda_{\bm{H}}}{\lambda_{\bm{W}}}\|\bm{H}\|_{F}^{2},
\end{equation}
and 
\begin{equation}\label{mseeq}
\rank(\bm{W})=\rank(\bm{H})\leq \rank(\widetilde{\bm{Y}}).
\end{equation}
\end{proposition}
\begin{proof}
	Any critical point $\left(\bm{W},\bm{H},\bm{b}\right)$ of \eqref{mainmse} satisfies:
	\begin{equation}\label{mser1}
		\nabla_{\bm{W}}
		f^{M}\left(\bm{W},\bm{H},\bm{b}\right)=\frac{1}{N}
		\left(\bm{WH-\widetilde{Y}}\right)\bm{H}^{\top}+\lambda_{\bm{W}}\bm{W}=\bm{0},
	\end{equation}
	\begin{equation}\label{mser2}
		\nabla_{\bm{H}}
		f^{M}\left(\bm{W},\bm{H},\bm{b}\right)=\frac{1}{N}
		\bm{W}^{\top}\left(\bm{WH-\widetilde{Y}}\right)+\lambda_{\bm{H}}\bm{H}=\bm{0}.
	\end{equation}
  By left multiplying \eqref{mser1} by $\bm{W}^{\top}$ on both sides, followed by right multiplying \eqref{mser2} by $\bm{H}^{\top}$ on both sides, and combining the equations together, we attain
$$
\lambda_{\bm{W}}\bm{W}^{\top}\bm{W}=\lambda_{\bm{H}}\bm{HH}^{\top}.
	$$
		The above identity is equivalent to
		\begin{equation}\label{eq:fd}
		\|(\bm{h}^{j})^{\top}\|_{2}^{2}=\frac{\lambda_{\bm{W}}}{\lambda_{\bm{H}}}\|\bm{w}_{j}\|_{2}^{2}, \quad j \in [K].
		\end{equation}
		Therefore, we have 		either
		$$
		\bm{w}_{j}=\bm{0},
		\quad 
		\bm{h}^{j}=\bm{0},
		$$ 
		or 
		$$
		\bm{w}_{j} \neq \bm{0},
		\quad 
		\bm{h}^{j} \neq \bm{0}.
	$$
	Moreover,
	$$
	\|\bm{W}\|_{F}^{2}=\trace\left(\bm{W}^{\top}\bm{W}\right)=\trace\left(\frac{\lambda_{\bm{H}}}{\lambda_{\bm{W}}}\bm{HH}^{\top}\right)=\frac{\lambda_{\bm{H}}}{\lambda_{\bm{W}}}\trace\left(\bm{HH}^{\top}\right)=\frac{\lambda_{\bm{H}}}{\lambda_{\bm{W}}}\|\bm{H}\|_{F}^{2}.
	$$
        The equations \eqref{mser1} and \eqref{mser2} are equivalent to
        \begin{equation}\label{eq:first}
            \bm{WHH}^{\top}+N\lambda_{\bm{W}}\bm{W}=\bm{\widetilde{Y}H}^{\top},
        \end{equation}
        \begin{equation}\label{eq:sec}
            \bm{W}^{\top}\bm{WH}+N\lambda_{\bm{H}}\bm{H}=\bm{W}^{\top}\bm{\widetilde{Y}}.
        \end{equation}
    We note that \eqref{eq:first} is equivalent to
    $$
        \bm{W}\left(\bm{HH}^{\top}+N\lambda_{\bm{W}}\bm{I}_{K}\right)=\bm{\widetilde{Y}H}^{\top}.
    $$
    The above identity, together with the fact 
    \begin{equation*}\label{msele}
        \rank\left( \bm{W}\left(\bm{HH}^{\top}+N\lambda_{\bm{W}}\bm{I}_{K}\right)\right)=\rank(\bm{W}),
    \end{equation*}
    implies that  
    \begin{equation}\label{mserl1}
        \rank(\bm{W})\leq \min\left\{\rank(\bm{\widetilde{Y}}),\rank(\bm{H}^{\top})\right\}.
    \end{equation}
    Similarly, it follows from \eqref{eq:sec} that 
       \begin{equation}\label{mserl2}
        \rank(\bm{H})\leq \min\left\{\rank(\bm{\widetilde{Y}}),\rank(\bm{W}^{\top})\right\}.
    \end{equation}
    Combining \eqref{mserl1} and \eqref{mserl2}, we have
   \begin{equation*}
       \rank(\bm{W})=\rank(\bm{H})\leq \rank(\widetilde{\bm{Y}}).
   \end{equation*}
\end{proof}


\begin{proposition}\label{prop2}
Suppose that   $(\bm{W},\bm{H},\bm{b})$ is a critical point  of \eqref{mainmse}.
The  SVD  of $\bm{W}$ is denoted by
$
\bm{W}=\bm{U_{W}\Sigma_{W}V_{W}}^{\top}.
$
Let
$\bm{\widetilde{W}}=\bm{WV_{W}}$ and  $\bm{\widetilde{H}} = {\bm{V_{W}}^{\top}}\bm{H}.$
Then 
$$
\bm{\widetilde{W}}=
\begin{pmatrix}
\bm{\widehat{W}} & \bm{0}
\end{pmatrix},\quad 
\bm{\widetilde{H}} = \begin{pmatrix}
\bm{\widehat{H}} \\ \bm{0}
\end{pmatrix},
$$
where the columns of $\bm{\widehat{W}}$ are orthogonal and the rows of $\bm{\widehat{H}}$  are orthogonal. The zeros in $\bm{\widetilde{W}}$ and
$\bm{\widetilde{H}}$
might or might not
exist, depending on the rank of $\bm{W}$ and $\bm{H}$. The point
$(\bm{\widetilde{W}},\bm{\widetilde{H}},\bm{b})$ is a critical point  of \eqref{mainmse}. 
Moreover,  $(\bm{W},\bm{H},\bm{b})$ and  $(\bm{\widetilde{W}},\bm{\widetilde{H}},\bm{b})$ have the same Hessian information.
\end{proposition}
\begin{proof}
Since $(\bm{W},\bm{H},\bm{b})$ is a critical point  of \eqref{mainmse},  we have
	$$
	\nabla_{\bm{W}}
	f^{M}\left(\bm{W},\bm{H},\bm{b}\right)=\bm{0} , \quad 
	\nabla_{\bm{H}}
	f^{M}\left(\bm{W},\bm{H},\bm{b}\right) =\bm{0}, \quad 
	\nabla_{\bm{b}}
	f^{M}\left(\bm{W},\bm{H},\bm{b}\right)= \bm{0},
	$$
	which are equivalent to
	\begin{equation}\label{oeq1}
	\frac{\lambda_{\bm{W}}}{\lambda_{\bm{H}}}\bm{WW}^{\top}\bm{W}+N\lambda_{\bm{W}}\bm{W}
	=
	\widetilde{\bm{Y}}
	\bm{H}^{\top},
	\end{equation}
	\begin{equation}\label{oeq2}
	\frac{\lambda_{\bm{H}}}{\lambda_{\bm{W}}}\bm{H}^{\top}\bm{HH}^{\top}+N\lambda_{\bm{H}}\bm{H}^{\top}
	=
\widetilde{\bm{Y}}^{\top}
	\bm{W},
	\end{equation}
	\begin{equation}\label{oeq3}
 	\frac{1}{N}
	\left(
	\bm{WH} + 	\bm{b 1}^{\top}_N-\bm{Y}
	\right)\bm{1}_N
 + \lambda_{\bm{b}}\bm{b}=\bm{0}.
	\end{equation}
	Using the SVD of $\bm{W}$, we have
		\begin{equation}\label{eq:ortho}
		\bm{\widetilde{W}}=\bm{WV_{W}}=\bm{U_{W}\Sigma_{W}}=
		\begin{pmatrix}
		\bm{\widehat{W}} & \bm{0}
		\end{pmatrix},
		\end{equation}
	where the columns of $\bm{\widehat{W}}$ are orthogonal and  $\bm{0}$ probably exists depending on $\rank(\bm{W})$.
	The condition
	$$\lambda_{\bm{W}}\bm{{W}}^{\top}\bm{{W}}=\lambda_{\bm{H}}\bm{{H}{H}}^{\top}$$  
	in  \eqref{eq:wh1} implies that 
	$$\lambda_{\bm{W}} \bm{\widetilde{W}}^{\top} \bm{\widetilde{W}}=\lambda_{\bm{H}} \bm{\widetilde{H}} \bm{\widetilde{H}}^{\top}.$$  	
	This indicates the rows of $\bm{\widetilde{H}} = \begin{pmatrix}
	\bm{\widehat{H}} \\ \bm{0}
	\end{pmatrix} $ are orthogonal.
	Moreover, 
	\begin{equation}\label{eq:oeqno}
	\bm{\widetilde{W}\widetilde{H}} =  {\bm{W}\bm{V_{W}V_{W}}^{\top}\bm{H}} =\bm{WH}.
	\end{equation}
	It follows from \eqref{oeq3} and \eqref{eq:oeqno}  that
	$$
	\frac{1}{N}\left(
	\bm{\widetilde{W}\widetilde{H}}+
	\bm{b 1}^{\top}_N-\bm{Y}
	\right) \bm{1}_{N} 
	+ \lambda_{\bm{b}}\bm{b}=\bm{0}.
	$$
	Multiplying  both sides of \eqref{oeq1} and  \eqref{oeq2} by  $\bm{V_{W}}$ , we attain 
	$$
	\frac{\lambda_{\bm{W}}}{\lambda_{\bm{H}}}\bm{\widetilde{W}\widetilde{W}}^{\top}\bm{\widetilde{W}}+N\lambda_{\bm{W}}\bm{\widetilde{W}}=
	\bm{\widetilde{Y}}
	\bm{\widetilde{H}}^{\top},
	\quad  
	\frac{\lambda_{\bm{H}}}{\lambda_{\bm{W}}}\bm{\widetilde{H}}^{\top}\bm{\widetilde{H}\widetilde{H}}^{\top}+N\lambda_{\bm{H}}\bm{\widetilde{H}}^{\top}
	=
	\bm{\widetilde{Y}}^{\top}
	\bm{\widetilde{W}}.
	$$
	Hence, $(\bm{\widetilde{W}},\bm{\widetilde{H}}, \bm{b})$ is also a critical point of \eqref{mainmse}.
 
Suppose that there  exists a non-zero vector $\bm{\alpha}\in\mathbb{R}^{K}$ such that 
\begin{equation}\label{eq:a1}
	\bm{W\alpha=0},\quad  \bm{\alpha}^{\top}\bm{H=0}.
\end{equation}
Then we  construct  a direction  $\bm{\Delta}$ for the critical point $(\bm{W},\bm{H},\bm{b})$ as
\begin{equation}\label{msed1}
	\bm{\Delta}=\left(\bm{\Delta_{W}},\bm{\Delta_{H}},\bm{\Delta_{b}}\right)=	
	\left(
	\left(\frac{\lambda_{\bm{H}}}{\lambda_{\bm{W}}}\right)^{\frac{1}{4}}\bm{u}\bm{\alpha}^{\top},
	\left(\frac{\lambda_{\bm{H}}}{\lambda_{\bm{W}}}\right)^{-\frac{1}{4}}\bm{\alpha } \bm{v}^{\top},
	\bm{0}
	\right)
\end{equation}
where  $\bm{u}\in\mathbb{R}^K$, $\bm{v}\in\mathbb{R}^N$ are unit vectors. Then
the Hessian  bilinear form of \eqref{mainmse} along the direction $\bm{\Delta}$ in \eqref{msed1}  is
\begin{align}        
	&\nabla^{2}f^{M}(\bm{W},\bm{H},\bm{b})\left[\bm{\Delta},\bm{\Delta}\right]\nonumber\\ 
	=&
	\frac{1}{N}	\|\bm{\Delta_{W}H+W\Delta_{H}}+\bm{\Delta_{b}1}_{N}^{\top}\|_{F}^{2} \nonumber\\
	&+2\trace\left(\frac{1}{N}\left(\bm{WH- \widetilde{Y}}\right)^{\top}\bm{\Delta_{W}\Delta_{H}}\right)\nonumber\\
	&+\lambda_{\bm{W}}\|\bm{\Delta_{W}}\|_{F}^{2}
	+\lambda_{\bm{H}}\|\bm{\Delta_{H}}\|_{F}^{2}, \label{msehess1}
	+\lambda_{\bm{b}}\|\bm{\Delta_{b}}\|_{2}^{2}.
\end{align}
Direct calculation shows that the direction 
$\bm{\Delta}$ in \eqref{msed1} satisfies the following properties
\begin{equation}\label{eq:delta11}
	\bm{\Delta_{W}H=0},\quad \bm{W\Delta_{H}=0},
	\quad \bm{\Delta_{b}1}_{N}^{\top}=\bm{0} ,\quad
	\|\bm{\Delta_{W}H+W\Delta_{H}}+\bm{\Delta_{b}1}_{N}^{\top}\|_{F}^{2}=0,
\end{equation}
\begin{equation}\label{eq:delta12}
\trace
\left(\frac{1}{N}\left(\bm{WH-\widetilde{Y}}\right)^{\top}\bm{\Delta_{W}\Delta_{H}}\right)
=
\trace
\left(\frac{
{\|\bm{\alpha}\|_{2}^{2}}
}{N}\left(\bm{WH-\widetilde{Y}}\right)^{\top}\bm{uv}^{\top}\right)
\end{equation}
and
\begin{equation}\label{eq:delta13}
	\lambda_{\bm{W}}\|\bm{\Delta_{W}}\|_{F}^{2}+\lambda_{\bm{H}}\|\bm{\Delta_{H}}\|_{F}^{2}+\lambda_{\bm{b}}\|\bm{\Delta_{b}}\|_{2}^{2}=2\|\bm{\alpha}\|_{2}^{2}\sqrt{\lambda_{\bm{W}}\lambda_{\bm{H}}}.
\end{equation}
We construct $\widetilde{\bm{\alpha}} = \bm{V_W}^{\top}\bm{\alpha}$ such that 
$$
\widetilde{\bm{W}} 
\widetilde{\bm{\alpha}} 
= 
\bm{W} \bm{V_W} \bm{V_W}^{\top}
\bm{\alpha} 
= 
\bm{W}\bm{\alpha}
=\bm{0},\quad 
\widetilde{\bm{\alpha }}^{\top}\widetilde{\bm{H}}=  \bm{\alpha}^{\top} \bm{V_W} \bm{V_W}^{\top} \bm{H} = \bm{\alpha}^{\top}  \bm{H}=\bm{0}.
$$
Similarly, the direction  given by 
\begin{equation*}\label{msed2}
\widetilde{\bm{\Delta}}
=
\left(
\bm{\widetilde{\Delta}_{\widetilde{\bm{W}}}} ,\bm{\widetilde{\Delta}_{\widetilde{\bm{H}}}},\bm{\widetilde{\Delta}}_{\bm{b}}
\right)
=\left(
\left(
\frac{\lambda_{\bm{H}}}{\lambda_{\bm{W}}}\right)^{\frac{1}{4}}\bm{u}\widetilde{\bm{\alpha }}^{\top},
\left(\frac{\lambda_{\bm{H}}}{\lambda_{\bm{W}}}\right)^{-\frac{1}{4}}\widetilde{\bm{\alpha }} \bm{v}^{\top},
\bm{0}
\right)
\end{equation*}
satisfies 
\begin{equation}\label{eq:delta21}
\|
\widetilde{\bm{\Delta}}_{\widetilde{\bm{W}}}  
\widetilde{\bm{H}}
+ 
\widetilde{\bm{W}} 
\widetilde{\bm{\Delta}}_{\widetilde{\bm{H}}}
+
\bm{\Delta_{b}1}_{N}^{\top}
\|_{F}^{2}
=0,
\end{equation}
\begin{equation}\label{eq:delta22}
\trace
\left(\frac{1}{N}
\left(\bm{\widetilde{W}\widetilde{H}-\widetilde{Y}}\right)^{\top}  
\widetilde{\bm{\Delta}}_{\widetilde{\bm{W}}}
\widetilde{\bm{\Delta}}_{\widetilde{\bm{H}}}
\right)
=
\trace
\left(\frac{
{\|\bm{\alpha}\|_{2}^{2}}
}{N}\left(\bm{\widetilde{W}\widetilde{H}-\widetilde{Y}}\right)^{\top}\bm{uv}^{\top}\right)
\end{equation}
and
\begin{equation}\label{eq:delta23}
\lambda_{\bm{W}}\|\widetilde{\bm{\Delta}}_{\widetilde{\bm{W}}}  \|_{F}^{2}
+
\lambda_{\bm{H}}\|\widetilde{\bm{\Delta}}_{\widetilde{\bm{H}}}  \|_{F}^{2}
+
\lambda_{\bm{b}}\|\widetilde{\bm{\Delta}}_{{\bm{b}}}  \|_{2}^{2}
=
2\|\bm{\alpha}\|_{2}^{2}\sqrt{\lambda_{\bm{W}}\lambda_{\bm{H}}}.
\end{equation}
Substituting
\eqref{eq:delta11}, \eqref{eq:delta12}, \eqref{eq:delta13} into \eqref{msehess1}, followed by  substituting \eqref{eq:delta21}, \eqref{eq:delta22}, \eqref{eq:delta23} into \eqref{msehess1}, together with \eqref{eq:oeqno}, we obtain 
\begin{equation}\label{eq:c2}
\nabla^{2}f^{M}\left(\bm{W},\bm{H},\bm{b}\right)
\left[
{\bm{\Delta}},{\bm{\Delta}}
\right]
=
\nabla^{2}f^{M}\left(\bm{\widetilde{W}},\bm{\widetilde{H}},\bm{b}\right)
\left[
\widetilde{\bm{\Delta}},\widetilde{\bm{\Delta}}
\right].
\end{equation}
Conversely, suppose that there exists a non-zero vector $\widetilde{\bm{\alpha}}\in\mathbb{R}^{K}$ satisfying
$$
\widetilde{\bm{W}} \widetilde{\bm{\alpha}} 
=\bm{0}
,\quad 
\widetilde{\bm{\alpha }}^{\top}\widetilde{\bm{H}}= \bm{0}.
$$
Then we can construct  $\bm{\alpha } =\bm{V_{W}} \widetilde{\bm{\alpha}}  $. Using a similar argument, we can show \eqref{eq:c2} holds.
In summary, the critical point $(\bm{W},\bm{H},\bm{b})$ and the critical point  $(\bm{\widetilde{W}},\bm{\widetilde{H}},\bm{b})$ have the same Hessian information.

\end{proof}

In accordance with Proposition \ref{mselm1} and Proposition \ref{prop2}, the following assumptions are applied to $(\bm{W},\bm{H},\bm{b})$  
in the rest of this section without loss of generality.
\begin{assumption}\label{ass1}
 We assume that the critical point $\left(\bm{W},\bm{H},\bm{b}\right)$ of \eqref{mainmse} satisfies the following two conditions:
\begin{enumerate}

	\item The matrix $\bm{W}$ and $\bm{H}$ are in the form
	\begin{equation}\label{o}
	\bm{W}=
	\begin{pmatrix}
	\bm{\widehat{W}} &  \bm{0}
	\end{pmatrix},\quad
	\bm{H}= \begin{pmatrix}
	\bm{\widehat{H}}  \\
	\bm{0} 
	\end{pmatrix},
	\end{equation}
	where the columns of $\bm{\widehat{W}}$ are orthogonal and the rows of $\bm{\widehat{H}}$ are orthogonal. The zeros in $\bm{W}$ and $\bm{H}$ exist when $\rank(\bm{W})=\rank(\bm{H})\leq K-1$. 
 \item For all $j \in [\rank(\bm{W})]$, $\bm{w}_j\neq \bm{0}$ and $\bm{h}^j\neq \bm{0}$.
\end{enumerate}
\end{assumption}
For any  critical point $\left(\bm{W},\bm{H},\bm{b}\right)$ of \eqref{mainmse} satisfying 
Assumption \ref{ass1},
we  decompose \eqref{oeq1}   for all  columns of $\bm{W}$ and rows of $\bm{H}$ into 
 \begin{align}\label{od1}
&\left(\frac{\lambda_{\bm{W}}}{\lambda_{\bm{H}}}\|\bm{w}_{j}\|^{2}_{2}+N\lambda_{\bm{W}}\right)\bm{w}_{j}
=\bm{\widetilde{Y}}(\bm{h}^{j})^{\top},
\nonumber \\ 
&\left(\frac{\lambda_{\bm{H}}}{\lambda_{\bm{W}}}\|(\bm{h}^{j})^{\top}\|^{2}_{2}+N\lambda_{\bm{H}}\right)(\bm{h}^{j})^{\top}=\bm{\widetilde{Y}}^{\top}\bm{w}_{j},\quad j \in [\rank(\bm{W})].
\end{align}
Plugging the equation \eqref{eq:fd}  into \eqref{od1} gives 
$$
\left(\frac{\sqrt{\lambda_{\bm{W}}}}{\sqrt{\lambda_{\bm{H}}}}\|\bm{w}_{j}\|_{2}^{2}
+N\sqrt{\lambda_{\bm{W}}\lambda_{\bm{H}}}\right)\frac{\bm{w}_{j}}{\|\bm{w}_{j}\|_{2}}
=\bm{\widetilde{Y}}\frac{(\bm{h}^{j})^{\top}}{\|(\bm{h}^{j})^{\top}\|_{2}},$$
$$
\left(\frac{\sqrt{\lambda_{\bm{W}}}}{\sqrt{\lambda_{\bm{H}}}}\|\bm{w}_{j}\|_{2}^{2}
+N\sqrt{\lambda_{\bm{W}}\lambda_{\bm{H}}}\right)\frac{(\bm{h}^{j})^{\top}}{\|(\bm{h}^{j})^{\top}\|_{2}}
=\bm{\widetilde{Y}}^{\top}\frac{\bm{w}_{j}}{\|\bm{w}_{j}\|_{2}} \quad  j \in [\rank(\bm{W})].
$$
Hence,   
\begin{equation}\label{mseeq:SVD1}
\sigma_{j}=\frac{\sqrt{\lambda_{\bm{W}}}}{\sqrt{\lambda_{\bm{H}}}}\|\bm{w}_{j}\|_{2}^{2}+N\sqrt{\lambda_{\bm{W}}\lambda_{\bm{H}}}
\end{equation}
are singular values of $\bm{\widetilde{Y}}$ with the left and right singular vectors 
\begin{equation}\label{mseeq:SVD2}
\bm{u}_{j}=\frac{\bm{w}_{j}}{\|\bm{w}\|_{2}},\quad \bm{v}_j=\frac{(\bm{h}^{j})^{\top}}{\|(\bm{h}^{j})^{\top}\|_{2}},\quad j \in [\rank(\bm{W})].
\end{equation}
On the other hand, it follows from \eqref{eq:fd}, \eqref{mseeq:SVD1} and \eqref{mseeq:SVD2} 
that
$$
\bm{w}_{j}\bm{h}^{j}
=\|\bm{w}_{j}\|_{2}^{2}\frac{\bm{w}_{j}}{\|\bm{w}_{j}\|_{2}}\frac{\bm{h}^{j}}{\|\bm{w}_{j}\|_{2}}
=\left(\sigma_{j}-N\sqrt{\lambda_{\bm{W}}\lambda_{\bm{H}}}\right)\bm{u}_{j}\bm{v}_j^T, \quad j \in [\rank(\bm{W})].
$$ 
This leads to
\begin{equation}\label{eq:WH}
\bm{WH} = \sum_{j=1}^{\rank(\bm{W})}
\bm{w}_{j}\bm{h}^{j}
=\sum_{j=1}^{\rank(\bm{W})}\left(\sigma_{j}-N\sqrt{\lambda_{\bm{W}}\lambda_{\bm{H}}}\right)\bm{u}_{j}\bm{v}_j^{\top}.
\end{equation}    

Since $\rank(\bm{W}) 
\leq \rank(\widetilde{\bm{Y}})$, we prove Theorem \ref{thm3} in two separate cases.
\begin{proof}[Proof of Theorem \ref{thm3}]
The first case is that 
$\rank(\bm{W})=\rank(\bm{\widetilde{Y}}).$
It follows from  \eqref{mseeq:SVD1}, \eqref{mseeq:SVD2} and \eqref{eq:WH}
that
\begin{equation*}\label{eq:dec}
\bm{WH-\widetilde{Y}}
=-\sum_{j=1}^{\rank(\bm{W})}N\sqrt{\lambda_{\bm{W}}\lambda_{\bm{H}}}\bm{u}_{j}\bm{v}_j^{\top},
\end{equation*}
which leads to 
\begin{equation}\label{eq:optimal1}
\|\bm{WH-\widetilde{Y}}\| = N\sqrt{\lambda_{\bm{W}}\lambda_{\bm{H}}}.
\end{equation} 
Since the equality \eqref{eq:optimal1} satisfies the condition in {Proposition  \ref{mselm2}},    
the critical point $\left(\bm{W},\bm{H},\bm{b}\right)$ is a global minimizer of \eqref{mainmse}.

The second case is that $\rank(\bm{W})<\rank(\widetilde{\bm{Y}})$. 
Then there  exist  singular values of $\bm{\widetilde{Y}}$ denoted by  $\sigma'_{i} $, which can not be covered by $\bm{u}_j$ and $\bm{v}_j$ 
in \eqref{mseeq:SVD2}, i.e.,
\begin{equation}
\langle \bm{u}'_{i} , \bm{u}_j \rangle=0,\quad
\langle \bm{v}'_{i} , \bm{v}_j \rangle=0,\quad j\in [\rank(\bm{W})]
\end{equation}
where $\bm{u}'_{i}$ and $\bm{v}'_{i}$ are
the  left and right singular value vectors to $\sigma'_{i}$. Therefore,
\begin{equation}\label{eq:uv}
\trace
\left(
\left(\bm{WH-\widetilde{Y}}\right)^{\top}
\bm{u}'_{i} (\bm{v}'_{i})^{\top}
\right)
=-\sigma'_{i}.
\end{equation}
We only need to consider that
the critical point $\left(\bm{W},\bm{H},\bm{b}\right)$ is not a global minimizer of \eqref{mainmse}.
Then Proposition  \ref{mselm2}
implies that there exists at least one singular value $\sigma'_{i^{\star}}$  satisfying \eqref{eq:uv} and 
\begin{equation}\label{eq:sigmastar2}
\sigma'_{i^{\star}}>N\sqrt{\lambda_{\bm{W}}\lambda_{\bm{H}}}.
\end{equation}
Let  $\bm{u}'_{i^{\star}}$ and $\bm{v}'_{i^{\star}}$ denote
the left and right singular value vectors 
corresponding to the singular
value $\sigma'_{i^{\star}}$. 
Since $\rank(\bm{W})<\rank(\bm{Y})$ implies that $\rank(\bm{W})<K$,
 there  exists a non-zero unit vector $\bm{\alpha}\in\mathbb{R}^{K}$ such that 
\begin{equation}\label{eq:a}
\bm{W\alpha=0},\quad  \bm{\alpha}^{\top}\bm{H=0}.
\end{equation}
Using $\bm{u}'_{i^{\star}}$, $\bm{v}'_{i^{\star}}$ and $\bm{\alpha}$ in \eqref{eq:a}, 
we  construct a strictly negative curvature direction for the critical point  $(\bm{W},\bm{H},\bm{b})$ as
 \begin{equation}\label{msed}
     \bm{\Delta}
     =
     \left(\bm{\Delta_{W}},\bm{\Delta_{H}},\bm{\Delta_{b}}\right)
     =	
     \left(
	\left(\frac{\lambda_{\bm{H}}}{\lambda_{\bm{W}}}\right)^{\frac{1}{4}}\bm{u}'_{i^{\star}}\bm{\alpha}^{\top},
	\left(\frac{\lambda_{\bm{H}}}{\lambda_{\bm{W}}}\right)^{-\frac{1}{4}}\bm{\alpha } (\bm{v}'_{i^{\star}})^{\top},
	\bm{0}
	\right).
 \end{equation}
It follows from \eqref{eq:a} and \eqref{msed} that 
 \begin{equation}\label{eq:newdelta11}
	\|\bm{\Delta_{W}H+W\Delta_{H}}+\bm{\Delta_{b}1}_{N}^{\top}\|_{F}^{2}=0
\end{equation}  
and 
\begin{equation}\label{eq:newdelta13}
	\lambda_{\bm{W}}\|\bm{\Delta_{W}}\|_{F}^{2}+\lambda_{\bm{H}}\|\bm{\Delta_{H}}\|_{F}^{2}+\lambda_{\bm{b}}\|\bm{\Delta_{b}}\|_{2}^{2}
	=2\sqrt{\lambda_{\bm{W}}\lambda_{\bm{H}}}.
\end{equation}
Substituting
\eqref{eq:uv}, 
\eqref{eq:newdelta11} and \eqref{eq:newdelta13} into \eqref{msehess1}, together with \eqref{eq:sigmastar2}, we attain
 $$
 \nabla^{2}f^{M}(\bm{W},\bm{H},\bm{b})\left[\bm{\Delta},\bm{\Delta}\right]
 =-\frac{2}{N}\left(\sigma'_{i^{\star}}-N\sqrt{\lambda_{\bm{W}}\lambda_{\bm{H}}}\right)<0.
 $$ 
The above inequality implies that the critical point $(\bm{W},\bm{H},\bm{b})$ of \eqref{mainmse}, which is not a local minimizer, is a strict saddle with negative curvature.

\end{proof}

\section*{Acknowledgments}

The author would like to thank Professors B. Dong
for the introduction to 
neural collapse. 
This work is supported by NSFC  under Grant number 	 12371101.

\bibliographystyle{plain}
\bibliography{references2}

\end{document}